%% file: main.tex
\documentclass{article}

\usepackage[preprint]{neurips_2022}

\usepackage[utf8]{inputenc} 
\usepackage[T1]{fontenc}    
\usepackage{hyperref}       
\usepackage{url}            
\usepackage{booktabs}       
\usepackage{amsfonts}       
\usepackage{nicefrac}       
\usepackage{microtype}      
\usepackage[usenames, dvipsnames]{xcolor}
\usepackage{graphicx}
\usepackage{amsmath}
\usepackage{amsthm}
\usepackage{amssymb}
\usepackage{booktabs}
\usepackage{algorithm}
\usepackage{algpseudocode}
\usepackage{dsfont}
\usepackage{caption}
\usepackage{subcaption}
\usepackage{multirow}
\usepackage{enumitem}
\usepackage{pifont}
\usepackage{adjustbox}
\usepackage{clipboard}
\usepackage{wrapfig}
\usepackage{soul}
\setcitestyle{square,sort,comma,numbers}

\setcitestyle{number}

\title{Structured Pruning is All You Need \\ for Pruning CNNs at Initialization}

\author{
Yaohui Cai \quad Weizhe Hua \quad Hongzheng Chen \quad \\
\textbf{G. Edward Suh \quad Christopher De Sa \quad Zhiru Zhang}\\
Cornell University\\
{\tt\small \{yc2632, wh399, hc676, gs272, cmd353, zhiruz\}@cornell.edu} \\
}

\input{commands}

\input{math_commands}

\begin{document}

\maketitle

\input{_s0_abstract}
\input{_s1_introduction}
\input{_s2_related}
\input{_s3_theorem}
\input{_s4_method}
\input{_s5_results}
\input{_s6_conclusion}

{
\bibliographystyle{ieee_fullname}
\bibliography{egbib}
}

\clearpage
\input{_s7_appendix}

\end{document}

%% file: commands.tex
\usepackage{xspace}

\iftrue

\newcommand{\todo}[1]{\textcolor{orange}{[\small TODO: ~#1]}}
\newcommand{\yc}[1]{\textcolor{orange}{[\small yc: ~#1]}}
\newcommand{\zz}[1]{\textcolor{blue}{[\small zz: ~#1]}}
\newcommand{\wh}[1]{\textcolor{red}{[\small wh: ~#1]}}
\newcommand{\hz}[1]{\textcolor{purple}{[\small hz: ~#1]}}

\else

\newcommand{\todo}[1]{\ignorespaces}
\newcommand{\yc}[1]{\ignorespaces}
\newcommand{\zz}[1]{\ignorespaces}
\newcommand{\wh}[1]{\ignorespaces}
\newcommand{\hz}[1]{\ignorespaces}
\fi

\def\cp {PreCrop\xspace}

\def\pai {PAI\xspace}
\def\goal {SynExp\xspace}
\def\ourthm {\goal Invariance Theorem\xspace}
\def\optflop {PreCrop-FLOPs\xspace}
\def\optparam {PreCrop-Params\xspace}

\def\reconf {PreConfig\xspace}

\newtheorem{theorem}{Theorem}[section]

\usepackage{tikz}
\usetikzlibrary{arrows,patterns}

%% file: math_commands.tex
\usepackage{amsmath,amsfonts,bm}

\def\Figref#1{Figure~\ref{#1}}

\def\Secref#1{Section~\ref{#1}}

\def\eqref#1{equation~\ref{#1}}
\def\Eqref#1{Equation~\ref{#1}}

\def\1{\bm{1}}

\DeclareMathAlphabet{\mathsfit}{\encodingdefault}{\sfdefault}{m}{sl}
\SetMathAlphabet{\mathsfit}{bold}{\encodingdefault}{\sfdefault}{bx}{n}

\newcommand{\E}{\mathbb{E}}

\newcommand{\R}{\mathbb{R}}



%% file: _s0_abstract.tex
\begin{abstract}
Pruning-at-initialization (\pai) proposes to prune the individual weights of the CNN before training, thus avoiding expensive fine-tuning or retraining of the pruned model.
While \pai shows promising results in reducing model size, the pruned model still requires unstructured sparse matrix computation, making it difficult to achieve wall-clock speedups.
In this work, we show theoretically and empirically that the accuracy of CNN models pruned by \pai methods only depends on the fraction of remaining parameters in each layer (i.e., \textit{layer-wise density}), regardless of the granularity of pruning.
We formulate the \pai problem as a convex optimization of our newly proposed expectation-based proxy for model accuracy, which leads to finding the optimal layer-wise density of that specific model.
Based on our formulation, we further propose a structured and hardware-friendly \pai method, named \cp, to prune or reconfigure CNNs in the channel dimension.
Our empirical results show that \cp achieves a higher accuracy than existing \pai methods on several modern CNN architectures, including ResNet, MobileNetV2, and EfficientNet for both CIFAR-10 and ImageNet.
\cp achieves an accuracy improvement of up to $2.7\%$ over the state-of-the-art \pai algorithm when pruning MobileNetV2 on ImageNet.
\cp also improves the accuracy of EfficientNetB0 by $0.3\%$ on ImageNet with only $80\%$ of the parameters and the same FLOPs.
\end{abstract}

%% file: _s1_introduction.tex
\section{Introduction}
\label{sec:intro}
Convolutional neural networks (CNNs) have achieved state-of-the-art accuracy in a wide range of machine learning (ML) applications. 
However, the massive computational and memory requirements of CNNs remain a major barrier to more widespread deployment on resource-limited edge and mobile devices.
This challenge has motivated a large body of research on CNN compression in order to
simplify the original model without significantly compromising accuracy. 

Weight pruning~\cite{lecun1990optimal, han2015deep, liu2018rethinking, frankle2018lottery, han2015learning} is one of the most extensively explored methods to reduce the computational and memory demands of CNNs.
Existing weight pruning approaches create a sparse CNN model by iteratively removing ineffective weights or activations and training the resulting sparse model.
Moreover, training-based pruning methods introduce additional hyperparameters, such as the learning rate for fine-tuning and the number of epochs before rewinding~\cite{renda2020comparing}, leading to a more complicated and less reproducible pruning process.
Among the different pruning techniques, training-based pruning usually enjoys the least accuracy degradation but at the cost of an expensive pruning procedure.

To minimize the cost of pruning, a new line of research proposes \emph{pruning-at-initialization (\pai)} \cite{lee2018snip, wang2020picking, tanaka2020pruning}, which identifies and prunes unimportant weights in a CNN right after initialization but before training.
As in training-based pruning methods, \pai evaluates the importance score of each individual weight and retains only a subset of weights by maximizing the sum of the importance scores of all remaining weights.
The compressed model is then trained using the same hyperparameters (e.g., weight decaying factor) for the same number of epochs as the baseline model.
Thus, the pruning and training of CNNs are cleanly decoupled, which greatly reduces the complexity of obtaining a pruned CNN model. 
Currently, SynFlow~\cite{tanaka2020pruning} is considered the state-of-the-art \pai technique --
it further eliminates the need for data in pruning as required in \cite{lee2018snip, wang2020picking} and achieves higher accuracy with the same number of parameters.

However, existing \pai methods mostly focus on fine-grained weight pruning, which removes individual weights from the CNN model without preserving any structure.
As a result, inference and training of the pruned model require sparse matrix computation, which is challenging to accelerate on commercially-available ML hardware such as GPUs and TPUs~\cite{google16tpu} that are optimized for dense computation.
According to a recent study~\cite{gale20sparseGPUkernel}, even with the NVIDIA cuSPARSE library, one can only achieve a meaningful speedup for sparse matrix multiplications on GPUs when the sparsity is over 98\%.
In practice, it is difficult for modern CNNs to shrink by more than 50$\times$ without a drastic degradation in accuracy~\cite{blalock2020state}.
Structural pruning patterns (e.g., pruning weights for the entire output channel) are needed to avoid irregularly sparse storage and computation, thus providing practical memory and computational savings.
Moreover, recent studies~\cite{su2020sanity, frankle2020pruning} also observe that randomly shuffling the binary weight mask of each layer or reinitializing all remaining weights does not affect the accuracy of the model compressed using existing \pai methods.
In this work, we first review the limitations of all previous \pai methods, 

Based on the observations, we hypothesize that existing \pai methods 
are only effective in determining the fraction of remaining weights in each layer, but fail to find a significant subset of weights.
We propose to use the expectation of the sum of importance scores of all weights, rather than the sum, as a proxy for the accuracy of the model, thus treating all weights in the same layer with  equal importance.
With our new proxy for accuracy named \textbf{SynExp} (Synaptic Expectation), we can formulate \pai as a convex optimization problem that directly solves the optimal fraction of remaining weights per layer (i.e., \textit{layer density}) subject to certain model size and/or FLOPs constraints.
We also prove a theorem that SynExp does not change the same as long as the layer-wise density remains the same, regardless of the granularity of pruning.
The theorem opens an important opportunity that coarse-grained \pai methods can achieve similar accuracy as their existing fine-grained counterparts such as SynFlow.
We demonstrate the efficacy of the proposed proxy through extensive empirical experiments.

We further propose \textbf{\cp}, a structured \pai that prunes CNN models at the channel level.
\cp can effectively reduce the model size and computational cost without loss of accuracy compared to fine-grained \pai methods, and more importantly, provide a wall-clock speedup on commodity hardware.
By allowing each layer to have more parameters than the baseline network, we 
are also able to \textit{reconfigure} the width dimension of the network with almost zero cost, 
which is termed \textbf{\reconf}.
Our empirical results show that the model after \reconf can achieve higher accuracy with fewer parameters and FLOPs than the baseline for a variety of modern CNNs.

We summarize our contributions as follows:
\begin{itemize}
    \item We propose to use the expectation of the sum of importance scores of all weights as a proxy for accuracy and formulate \pai as a \goal optimization problem constrained by the model size and/or FLOPs.
    We also prove that the accuracy of the CNN model pruned by solving the constrained optimization is independent of the pruning granularity.
    
    \item We introduce \cp 
    to prune CNNs at the channel level based on the proposed \goal optimization.
    Our empirical study demonstrates that models pruned using \cp achieve similar or better accuracy compared to the state-of-the-art unstructured PAI approaches while preserving regularity.
    Compared to SynFlow, \cp achieves 2.7\% and 0.9\% higher accuracy on MobileNetV2 and EfficientNet on ImageNet with fewer parameters and FLOPs.
    \item We show that \reconf can be used to optimize the width of each layer in the network with almost zero cost (e.g., the search can be done within one second on a CPU).
    Compared to the original model, \reconf achieves 0.3\% accuracy improvement with 20\% fewer parameters and the same FLOPs for EfficientNet and MobileNetV2 on ImageNet.
    
\end{itemize}

%% file: _s2_related.tex
\section{Related Work}
\textbf{Model Compression in General}
can reduce the computational cost of large networks to ease their deployment in resource-constrained devices.
Besides pruning, quantization~\cite{dong2019hawq, zhou2016dorefa, jacob2018quantization}, NAS~\cite{zoph2016neural, tan2019efficientnet},  
and distillation \cite{hinton2015distilling, yin2020dreaming} are also commonly used to improve the model efficiency.

\textbf{Training-Based Pruning} uses various heuristic criteria to prune unimportant weights. They typically require an iterative training-prune-retrain process where the pruning stage is intertwined with the training stage, which may increase the overall training cost by several folds. 
Because pruning aims to reduce parameters, 
the FLOPs reduction is usually less significant~\cite{liu2017learning, frankle2018lottery}.

Existing training-based pruning methods can be either unstructured~\cite{han2015deep,lecun1990optimal} or structured~\cite{he2017channel, luo2017thinet}, depending on  the granularity and regularity of the pruning scheme. Training-based unstructured pruning usually provides a better accuracy-size trade-off while structured pruning can achieve a more practical speedup and compression without  special support of custom hardware. 

\textbf{(Unstructured) Pruning-at-Initialization (\pai)}~\cite{lee2018snip, wang2020picking, tanaka2020pruning} 
provides a promising approach to mitigating the high cost of training-based pruning. 

They can identify and prune unimportant weights right after initialization and before the training starts.
Related to these efforts, authors of \cite{frankle2020pruning} and \cite{su2020sanity} independently find that for all existing \pai methods, 
randomly shuffling the weight mask within a layer or reinitializing all the weights in the network does not cause any accuracy degradation.

\textbf{Neural Architecture Search (NAS)}
~\cite{zoph2016neural,wan2020fbnetv2} 
automatically searches over a large set of candidate models to achieve the optimal accuracy-computation trade-off.
The search space of NAS usually includes width, depth, resolution, and choice of building blocks.
However, existing approaches can only search in a small subset of the possible channel width configurations due to the cost.
The cost for NAS is also orders of magnitude higher than training a model.
Some NAS algorithms~\cite{abdelfattah2021zero, zhou2020econas} use a cheap proxy instead of training the whole network, 
but an expansive reinforcement learning~\cite{zoph2016neural} or evolutionary algorithm~\cite{real2019regularized} is still used to predict a good network.

%% file: _s3_theorem.tex
\section{Pruning-at-initialization via \goal Optimization}

In this section, we first discuss the preliminaries and limitations of existing \pai methods.
We propose a new proxy for the accuracy of the compressed model to overcome these limitations.
With the proposed proxy, we formulate the \pai problem into a convex optimization problem.

\subsection{Preliminaries and Limitations of \pai}
\label{sec:background}
\textbf{Preliminaries.} 
\pai aims to prune neural networks after initialization but before training to avoid the time-consuming \emph{training-pruning-retraining} process.
Prior to training, \pai uses the gradients with respect to the weights to estimate the importance of individual weights, which requires forward and backward propagations.
Weights ($W$) with smaller importance scores are pruned by setting the corresponding entries in the binary weight mask ($M$) to zero.
Existing \pai methods, such as SNIP~\cite{lee2018snip}, GraSP~\cite{wang2020picking}, and SynFlow~\cite{tanaka2020pruning} mainly explore different methods to estimate the importance of individual weights.
Single-shot \pai algorithms, such as SNIP and GraSP,
prune the model to the desired sparsity in a single pass.
Alternatively, SynFlow, which represents the state-of-the-art \pai algorithm, repeats the process of pruning a small fraction of weights and re-evaluating the importance scores until the desired pruning rate is reached.
Through the iterative process, the importance of weights can be estimated more accurately.

Specifically, the importance score used in SynFlow can be formulated as:
\begin{equation}
\label{eq:synflow}
    \mathcal{S}_{\text{SF}}(W^{l}_{ij}) = 
    \left[\mathds{1}^T\prod_{k=l+1}^{N}\left|W^{k}\odot M^{k}\right|\right]_i 
    \left|W^{l}_{ij} M^{l}_{ij}\right| 
    \left[\prod_{k=1}^{l-1} \left|W^{k}\odot M^{k}\right|\mathds{1}\right]_j,
    \end{equation}
where $N$ is the number of layers,
$W^{l}$ and $M^{l}$ are the weight and weight mask of the $l$-th layer,
$\mathcal{S}_{\text{SF}}(W^{l}_{ij})$ is the SynFlow score for a single weight $W^{l}_{ij}$,
$\odot$ denotes the Hadamard product,
$|\cdot|$ is element-wise absolute operation,
and $\mathds{1}$ is an all-one vector.
It is worth noting that none of the data or labels is used to compute the importance score, thus making SynFlow a data-agnostic algorithm.

\textbf{Limitations.}
As pointed out by \cite{su2020sanity, frankle2020pruning}, randomly shuffling the weight mask $W$ of each layer or reinitializing all the weights $M$ does not affect the final accuracy of models compressed using existing \pai methods.
In addition, they show that given the same layer-wise density (i.e., the fraction of remaining weights in each layer), the pruned models will have similar accuracy.
The observations suggest that even though the existing \pai algorithms try to identify less important weights,
which weights to prune is not important for accuracy.

All previous \pai methods use the \textit{sum} of importance scores of the remaining weights as a proxy for model accuracy, which is identical to the training-based pruning~\cite{han2015deep,lecun1990optimal}.
\pai obtains the binary weight mask by maximizing the proxy as follows
\begin{equation}
\label{eq:trad}
\mbox{maximize } \sum_{l=1}^N S^{l} \cdot M^{l} \quad \mbox{ over } M \quad \mbox{ subject to } \sum_{l=1}^N \|M^{l}\|_0 \le B_{\text{params}} \ , 
\end{equation}
where $N$ is the number of layers in the network, $S^{l}$ is the score matrix in the $l$-th layer, $M^{l}$ is the binary weight mask in the $l$-th layer, $\|\cdot\|_0$ is the number of nonzero entries in a matrix, and $B_{\text{params}}$ is a pre-defined model size constraint.

Regardless of the specific \pai importance score chosen, a subset of weights is determined to be more important than the other weights, which contradicts the observation that random shuffling does not affect accuracy.
Instead, we propose a new accuracy proxy for \pai to address the limitations.

\subsection{\ourthm}
\label{sec:theorem}
Inspired by the observations of previous \pai methods, we conjecture that a proxy for the accuracy of the model pruned using a \pai method should satisfy the following two properties:
\begin{itemize}
    \item [1.] The pruning decision (i.e., weight mask $M$) can be made \textit{before} the model is initialized. 
    \item [2.] Maximization of the proxy should result in optimal \textit{layer-wise density}, not pruning decisions for individual weights.
    
\end{itemize}

For random pruning before initialization, given a fixed density $p_l$ for each layer, the weight matrix $W_l$ and the binary mask matrix $M_l$ of that layer can be considered two random variables.
The binary weight mask is applied to the weight matrix element-wisely as $W_l \odot M_l$, where $\odot$ represents Hadamard product.
The weight matrix of layer $l$ (i.e., $W^l$) contains $\alpha_l$ parameters.
Each individual weight in layer $l$ is sampled independently from a given distribution $D^l$.
Suppose $A^l=\{M, M_i \in \{0, 1\}\ \forall 1\le i\le \alpha_l, \sum_i M_i = p_l \times \alpha_l\}$ is the set of all possible binary matrices with the same shape as the weight matrix $W^l$ that satisfy the layer-wise density ($p_l$) constraint.
Then, the random weight mask $M^l$ for layer $l$ is sampled uniformly from $A^l$.

Let $M=\{M^l, \forall 1\le l\le N\}$ and $W=\{W^l, \forall 1\le l\le N\}$ be the weights and masks of all $N$ layers in the network, respectively.
The observations in Section~\ref{sec:background} indicate that any instantiation of the two random variables $M$ and $W$ results in similar final accuracy of the pruned model.
However, the different instantiations do change the proxy value for the model accuracy in existing \pai methods.
For example, the SynFlow score in \Eqref{eq:trad} changes under different instances of $M$ and $W$.
Therefore, we propose a new proxy that is invariant to the instantiations of $M$ and $W$ for the model accuracy in the context of \pai~--- 
the \textit{expectation} of the sum of the importance scores of all unpruned (i.e., remaining) weights.
The proposed proxy can be formulated as follows:
\begin{equation}
\begin{aligned}
\label{eq:expectation}
\mbox{maximize} \mathop{\mathbb{E}}_{M, W}[\mathcal{S}] 
=\mathop{\mathbb{E}}_{M, W}\left[\sum_{l=1}^N S^{l} \cdot M^{l}\right] \quad 
\mbox{ over } p_l \quad \mbox{ subject to }
\sum_{l=1}^N \alpha_l \cdot p_l \le B_{\text{params}} \ , 
\end{aligned}
\end{equation}
where $p_l$ represents layer-wise density of layer $l$,
$\alpha_l$ is the number of parameters in layer $l$,
$\mathop{\mathbb{E}}_{M, W}[\mathcal{S}]$ stands for the expectation of the importance score $\mathcal{S}$ over random weight $W$, and binary random mask $M$.
In this new formulation, the layer-wise density $p_l$ is optimized to maximize the proposed proxy for model accuracy.

In order to evaluate the expectation before weight initialization, we adopt the importance metric proposed by SynFlow, i.e., replacing $\mathcal{S}$ in \Eqref{eq:expectation} with $\mathcal{S}_\text{SF}$ in \Eqref{eq:synflow}.
As a result, we can compute the expectation analytically without forward or backward propagations.
This new expectation-based proxy is referred to as \textit{\goal} (Synaptic Expectation).
We further prove \goal is invariant to the granularity of \pai in the 
\ourthm. 
The detailed proof can be found in Appendix~\ref{app:proof}.

\begin{theorem}
Given a specific CNN architecture, the \goal ($\E_{[M, W]} [\mathcal{S}_\text{SF}]$) of any randomly compressed model with the same layer-wise density $p_l$ is a constant, independent of the pruning granularity.
The constant \goal equals to:
\begin{equation}
\label{eq:goal}
\mathop{\E}_{M, W}[\mathcal{S}_\text{SF}] 
= N C_{N+1}\prod_{l=1}^N (p_l C_{l} \cdot \E_{x\sim \mathcal{D}}[|x|]) \ ,
\end{equation}
where $N$ is the number of layers in the network,
$\E_{x\sim \mathcal{D}}[|x|]$ is the expectation of magnitude of distribution $\mathcal{D}$,
$C_l$ is the input channel size of layer $l$ and is also the output channel size of $l-1$,
and $p_l=\frac{1}{\alpha_l} \|M_l\|_0$ is the layer-wise density.
\end{theorem}
In Equation~\ref{eq:goal}, $N$ and $C^l$ are all hyperparameters of the CNN architecture and can be considered constants.
$\E_{|\mathcal{D}^l|}$ is also a constant under a particular distribution $\mathcal{D}^l$.
The layer-wise density $p_l$ is the only variable that needs to be solved in the equation.
Thus, \goal satisfies both of the aforementioned properties: 1) pruning is done prior to the weight initialization; 2) the layer-wise density can be directly optimized.
Furthermore, Theorem 1 also shows that the granularity of pruning has no impact on the proposed \goal metric.
In other words, the CNN models compressed using unstructured and structured pruning methods will have similar accuracy.

\input{figures/thm_compare}
We empirically verify Theorem 1 by randomly pruning each layer of a CNN with different pruning granularity but the same layer-wise density ($p_l$).
In this empirical study, we perform random pruning with three different granularities (i.e., weight, filter, and channel) to achieve the desired layer-wise density obtained from solving Equation~\ref{eq:expectation}.
For weight and filter pruning,
randomly pruning each layer to match the layer-wise density $p_l$ occasionally detaches some weights from the network, especially when the density is low.
The detached weights do not contribute to the prediction but are counted as remaining parameters.
Thus, we remove the detached weights for a fair comparison following~\cite{vysogorets2021connectivity}.
For channel pruning, 
it is non-trivial to achieve the given layer-wise density while satisfying the constraint that the output channel size of the previous layer should be equal to the input channel size of the next layer.
Therefore, we use \cp proposed in \Secref{sec:prune}.
As shown in \Figref{fig:thm_comp}, 
random pruning with different granularity can obtain similar accuracy compared to SynFlow,
as long as the layer-wise density remains the same.
The empirical results are consistent with Theorem 1 and also demonstrate the efficacy of the proposed \goal metric.
We include more empirical results for different CNN architectures and different importance scores in Appendix~\ref{app:val}.

\subsection{Optimizing \goal}
As discussed in \Secref{sec:theorem}, only the layer-wise density matters for our proposed \goal approach.
Here, we show how to obtain the layer-wise density in \Eqref{eq:expectation} that maximizes \goal under model size and/or FLOPs constraints.

\subsubsection{Optimizing \goal with Parameter Count Constraint}
\label{sec:opt_params}
Given that the goal of \pai is to reduce the size of the model, we need to add an additional constraint on the total number of parameters $B_\text{params}$ (i.e., parameter constraint), where $B_\text{params}$ is typically greater than zero and less than the number of parameters in the original network.
Since layer-wise density $p_l$ is the only variable in \Eqref{eq:expectation}, we can simplify the equation by removing all other constant terms, as follows:
\begin{align}
\label{eq:opt_params}
\begin{split}
\mbox{maximize } \sum_{l=1}^N \log p_l \quad \mbox{ over } p_l \quad
\text{ subject to } \
& \sum_{l=1}^N \alpha_l \cdot p_l \le B_{\text{params}} \ , \\
& 0<p_l \le 1, \forall 1\le l \le N \ , 
\end{split}
\end{align}
where $\alpha_l$ is the number of parameters in layer $l$.

\Eqref{eq:opt_params} is a convex optimization problem that can be solved analytically\textsuperscript{\ref{ft1}}.
We compare the layer-wise density derived from solving Equation~\ref{eq:opt_params} with the density obtained using SynFlow.
As shown in Figure~\ref{fig:spar_comp}, the layer-wise density obtained by both approaches are nearly identical, where our new formulation gets rid of the iterative re-evaluation of SynFlow scores and the pruning process in SynFlow.
It is also worth noting that the proposed method finds the optimal layer-wise density even before the network is initialized.
\input{figures/spar_comp}

\subsubsection{Optimizing \goal with Parameter Count and FLOPs Constraints}
As discussed in \Secref{sec:opt_params},
we can formulate \pai as a simple convex optimization problem with a constraint on the model size.
However, the number of parameters does not always reflect the performance (e.g., throughput) of the CNN model.
In many cases, CNN models are compute-bound on many commodity hardwares~\cite{google16tpu, harish2007accelerating}. 
Therefore, we propose to also introduce a FLOPs constraint in our formulation.

The FLOPs saved in existing \pai algorithms specified in \Eqref{eq:trad} come from pruning the weights in the CNN model.
In other words, given a parameter constraint, the FLOP count of the pruned model is also determined.
It is not straightforward to introduce a FLOP constraint for each layer in the model, because the correspondence between the number of parameters and FLOPs varies across different layers.
Therefore, none of the existing \pai methods can be directly used to bound the FLOPs of CNN models.
Since the weights in the same layer are associated with the same FLOP count, we can directly incorporate the constraint on FLOPs $B_{\text{FLOPs}}$ (i.e., FLOPs constraint) into the convex optimization problem as follows:
\begin{align}
\label{eq:opt_both}
\begin{split}
\mbox{maximize }\sum_{l=1}^N \log p_l \quad \mbox{ over } p_l  \quad
\text{ subject to } \
& \sum_{l=1}^N \alpha_l \cdot p_l \le B_{\text{params}}, \sum_{l=1}^N \beta_l \cdot p_l \le B_{\text{FLOPs}} \\
& 0<p_l \le 1, \forall 1\le l \le N,
\end{split}
\end{align}
where $\beta_l$ in the number of FLOPs in the $l^{th}$ layer.

\input{figures/pareto}
Since the additional FLOPs constraint is linear, 
the optimization problem in Equation~\ref{eq:opt_both} remains convex and has an analytical solution\footnote{\label{ft1}We include analytical solutions for \Eqref{eq:opt_params} and \Eqref{eq:opt_both} in Appendix~\ref{app:solve} for completeness.}.
By solving \goal optimization with a fixed $B_\text{params}$ but different $B_\text{FLOPs}$, we can obtain the layer-wise density for various models that have the same number of parameters but different FLOPs.
Then, we perform random weight pruning on the CNN model to achieve the desired layer-wise density.
We compare the proposed \goal optimization (denoted as Ours) with other popular \pai methods.  
As depicted in \Figref{fig:pareto}, given a fixed model size ($1.5\times 10^4$ in the figure), our method can be used to generate a Pareto Frontier that spans the spectrum of FLOPs, while other methods can only have a fixed FLOPs.
Our method dominates all other methods in terms of both accuracy and FLOPs reduction.

%% file: figures/thm_compare.tex
\begin{wrapfigure}{R}{0.5\textwidth}
\centering
\vspace{-8pt}
\begin{subfigure}[H]{0.5\textwidth}
    \centering
    \includegraphics[width=\linewidth]{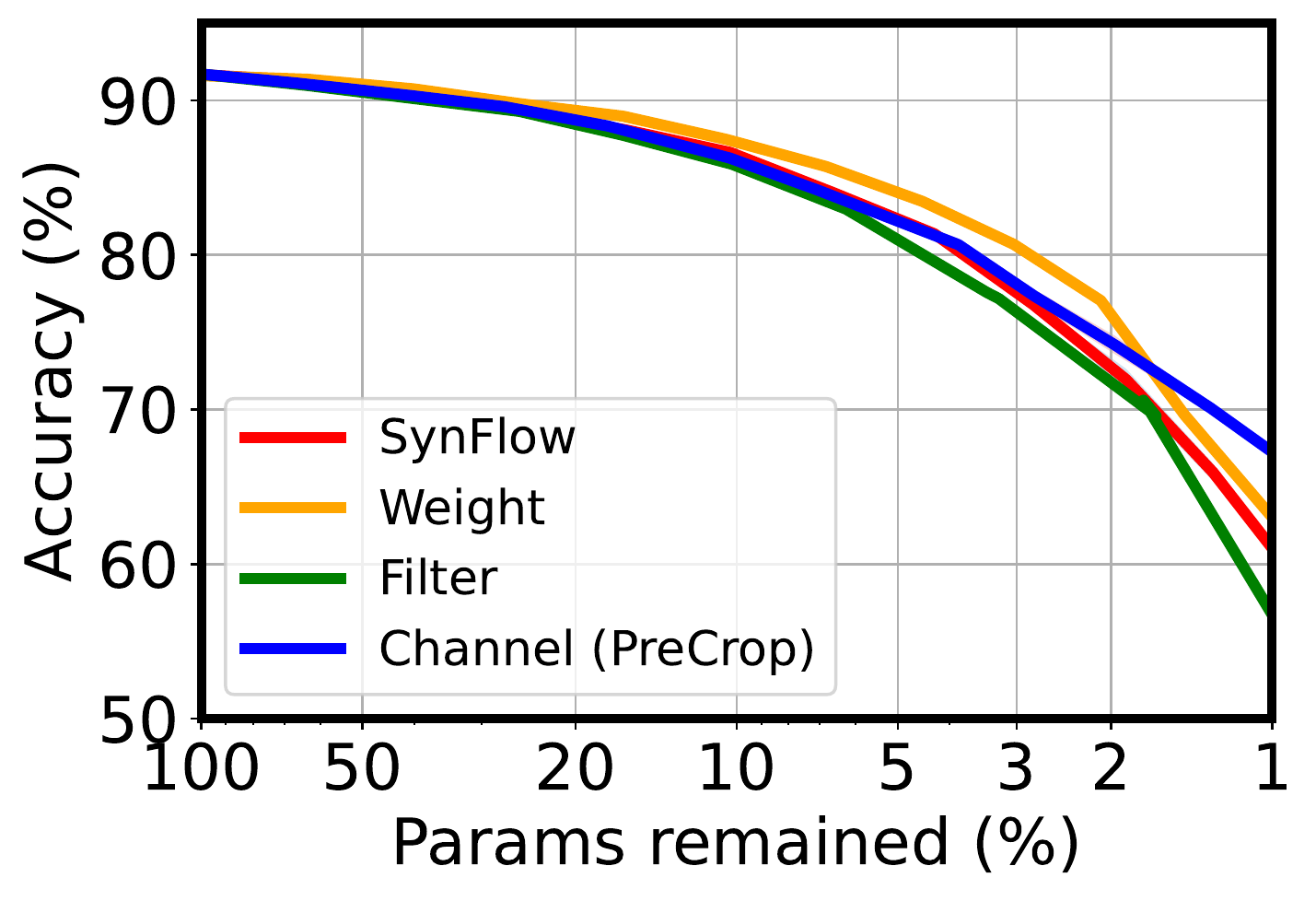}
\end{subfigure}
\hfill
\centering
\vspace{-5pt}
\caption{
Comparison of the performance using different pruning granularities on ResNet20 using CIFAR-10.
}
\vspace{-5pt}
\label{fig:thm_comp}
\end{wrapfigure}

%% file: figures/spar_comp.tex
\begin{figure}[t]
\vspace{-15pt}
\centering
\begin{subfigure}[H]{\textwidth}
    \centering
    \includegraphics[width=0.85\linewidth]{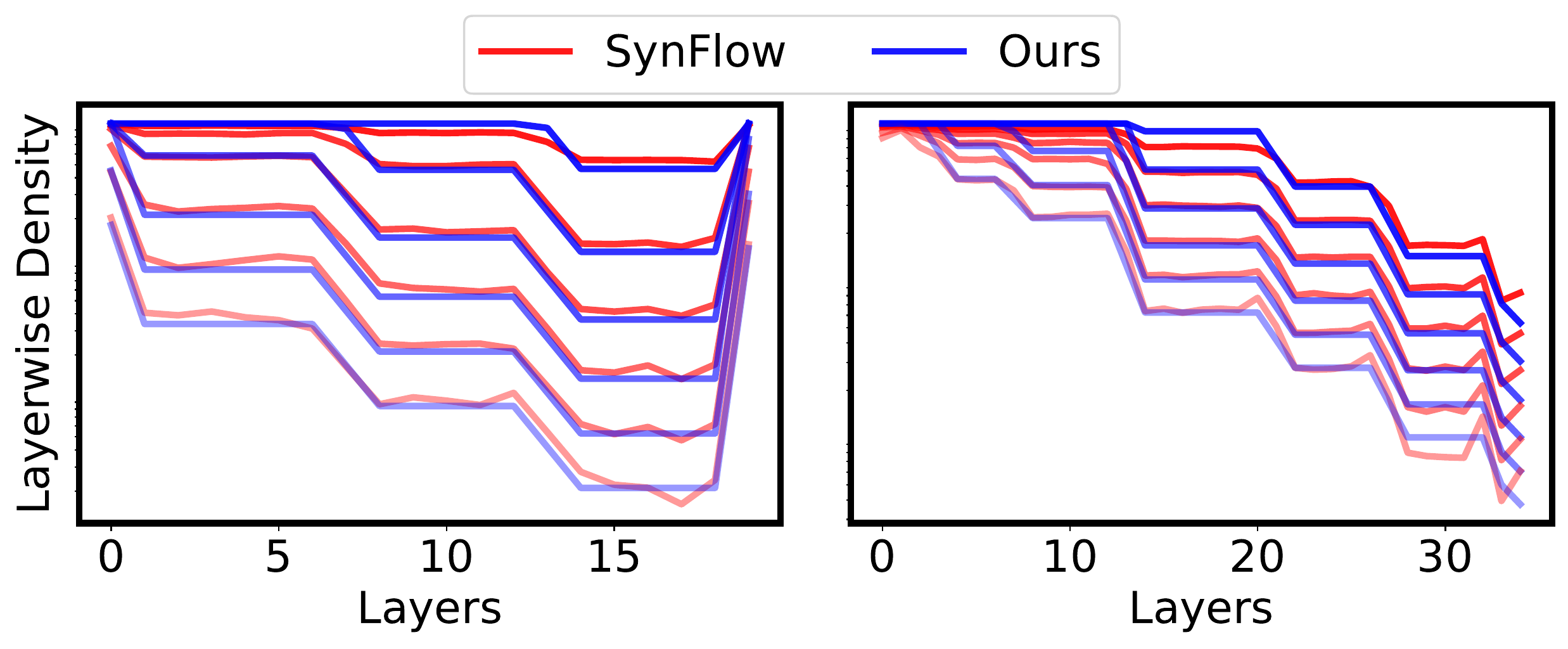}
\end{subfigure}
\vspace{-.6cm}\\
\subfloat[ResNet-20]{\hspace{.42\linewidth}}
\subfloat[MobileNetV2]{\hspace{.4\linewidth}}
\caption{
Comparison of the layer-wise densities obtained by \goal optimization with parameter count constraint and SynFlow.
Higher transparency means that the problem is constrained by a smaller parameter count.
}
\vspace{-5pt}
\label{fig:spar_comp}
\end{figure}

%% file: figures/pareto.tex
\begin{wrapfigure}{R}{0.45\textwidth}
\centering
\vspace{-15pt}
\begin{subfigure}[H]{0.45\textwidth}
    \centering
    \includegraphics[width=\linewidth]{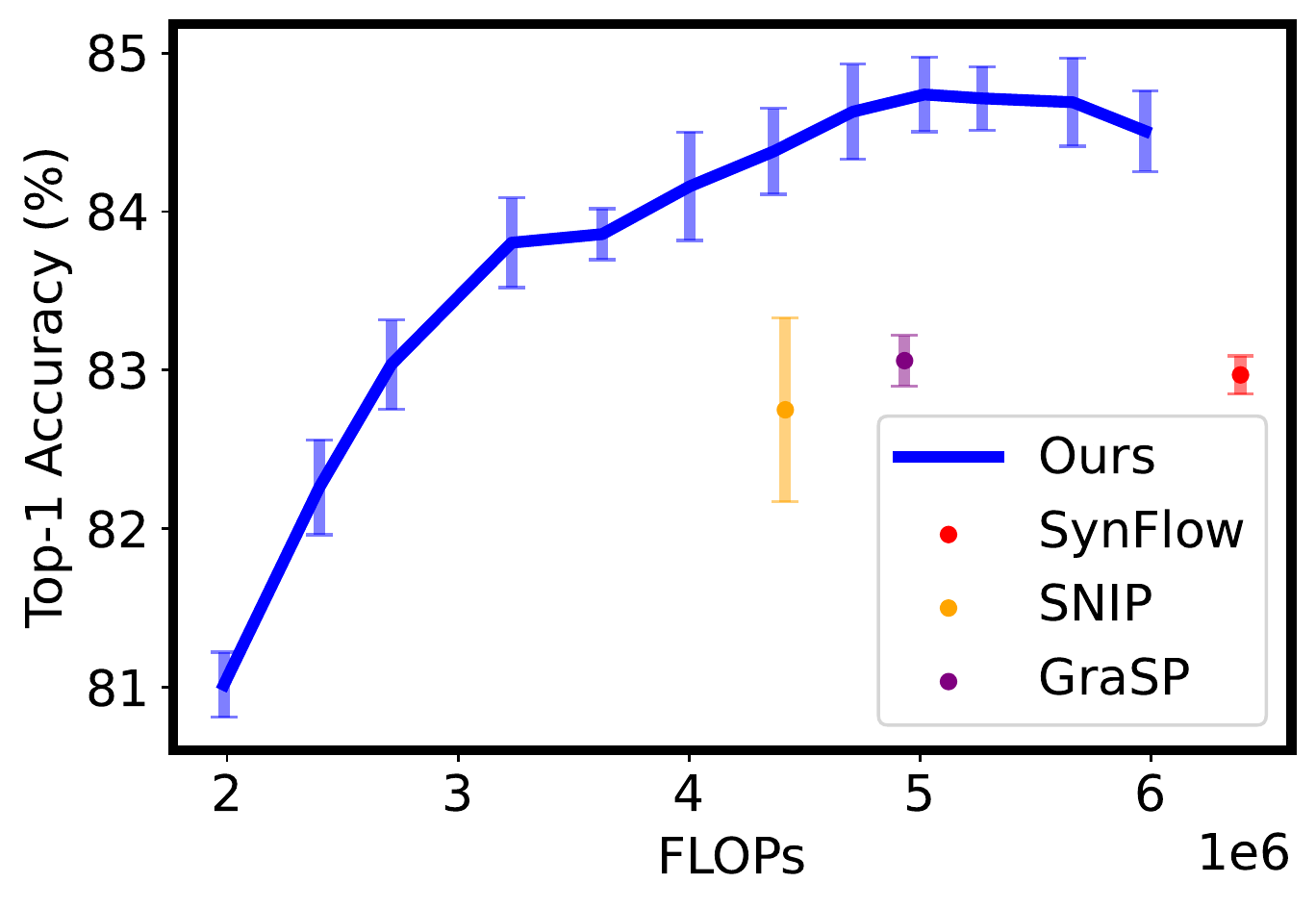}

\end{subfigure}
\hfill
\centering
\caption{
Comparison of our method with other \pai methods --- \small{we repeat the experiment using ResNet-20 on CIFAR-10 five times and report the mean and variance (error bar) of the accuracy. All the models in the figures have $1.5 \times 10^4$ parameters.}
}
\vspace{-15pt}
\label{fig:pareto}
\end{wrapfigure}

%% file: _s4_method.tex
\section{Structured Pruning-at-Initialization}

The \ourthm shows that pruning granularity of \pai methods should not affect the accuracy of the pruned model.
Channel pruning, which prunes the weights of the CNN at the output channel granularity, is considered the most coarse-grained and hardware-friendly pruning technique, 
Therefore, applying the proposed \pai method for channel pruning can avoid both complicated retraining/re-tuning procedures and irregular computations.
In this section, we propose a structured \pai method for channel pruning, named \cp, to prune CNNs in the channel dimension.
In addition, we propose a variety of \cp with relaxed density constraints to reconfigure the width of each layer in the CNN model, which is called PreConfig.

\subsection{\cp}
\input{figures/precropping}
\label{sec:cp}
Applying the proposed \pai method to channel pruning requires a two-step procedure.
First, the layer-wise density $p_l$ is obtained by solving the optimization problem shown in Equation~\ref{eq:opt_params} or \ref{eq:opt_both}.
Second, we need to decide how many output channels of each layer should be pruned to satisfy the layer-wise density.
However, it is not straightforward to compress each layer to match a given layer-wise density due to the additional constraint that the number of output channels of the current layer must match the number of input channels of the next layer.

We introduce \cp, 
which compresses each layer to meet the desired layer-wise density.
Let $C_l$ and $C_{l+1}$ be the number of input channels of layer $l$ and $l+1$, respectively.
$C_{l+1}$ also means the number of output channels of layer $l$.
For layers with no residual connections, 
the number of output channels of layer $l$ is reduced to 
$\left\lfloor\sqrt{p_l} \cdot C_{l+1}\right\rfloor$.
The number of input channels of layer $l+1$ needs to match the number of output channels of layer $l$, which is also reduced to $\left\lfloor\sqrt{p_{l}} \cdot C_{l+1}\right\rfloor$.
Therefore, the actual density of layer $l$ after \cp is $\sqrt{p_{l-1} \cdot p_l}$ instead of $p_l$.
We empirically find that $\sqrt{p_{l-1}\cdot p_l}$ is close enough to $p_l$ because the neighboring layers have similar layer-wise densities.
Alternatively, it is possible to obtain the exact layer-wise density $p$ by only reducing the number of input or output channels of a layer.
However, this approach leads to a significant drop in accuracy,
because the number of the input and output channels can change dramatically (e.g., $p_lC_l\ll C_{l+1}$ or $C_l \gg p_l C_{l+1}$).
This causes the shape of the feature map changes dramatically in adjacent layers, resulting in information loss.

For layers with residual connections, \Figref{fig:precrop} depicts an approach to circumvent the constraint on the number of channels of adjacent layers.
We can reduce the number of input and output channels of layer $l$ from $C_l$ and $C_{l+1}$ to $\sqrt{p_l}C_l$ and $\sqrt{p_l}C_{l+1}$, respectively.
In this way, the density of each layer can match the given layer-wise density obtained from the proposed \pai method.
Since the output of layer $l$ needs to be added element-wisely with the original input to layer $l$, the output of layer $l$ is padded with zero-valued channels to match the shape of the original input.
In our implementation, 
we simply add the output of layer $l$ to the first $\sqrt{p_l}C_{l+1}$ channels of the original input to layer $l$, 
thus requiring no extra memory or computation for zero padding.
\cp eliminates the requirement for sparse computation in existing \pai methods and thus can be used to accelerate both training and inference of the pruned models.

\subsection{\reconf: \cp with Relaxed Density Constraint}
\label{sec:reconf}

\cp uses the layer-wise density obtained from solving the convex optimization problem, which is always less than $1$ following the common setting for pruning (i.e., $p_l\le 1$).
However, this constraint on layer-wise density is not necessary for our method since we can increase the number of channels (i.e., expand the width of the layer) before initialization.
By solving the problem in \Eqref{eq:opt_both} without the constraint $p_l \le 1$,
we can expand the layers with a density greater than 1 ($p_l>1$) and prune the layers with a density less than 1 ($p_l<1$).
We call this variant of \cp as \reconf (PreCrop-Reconfigure).
If we set $B_{\text{params}}$ and $B_{\text{FLOPs}}$ to be the same as the original network, we can essentially reconfigure the width of each layer of a given network architecture under certain constraints on model size and FLOPs.

The width of each layer in a CNN is usually designed manually, which often relies on extensive experience and intuition.
Using \reconf, we can determine the width of each layer in the network to achieve a better cost-accuracy tradeoff.
\reconf can also be used as an or a part of an ultra-fast NAS.
Compared to NAS, which typically searches on the width, depth, resolution, and choice of building blocks, \reconf only changes the width.
Nonetheless, \reconf only requires a minimum amount of time and computation compared to NAS methods; it only needs to solve a relatively small convex optimization problem, which can be solved in a second on a CPU.
\label{sec:prune}

%% file: figures/precropping.tex
\begin{figure}[t]
\vspace{-15pt}
\centering
\begin{subfigure}[H]{\textwidth}
    \centering
    \includegraphics[width=0.8\linewidth]{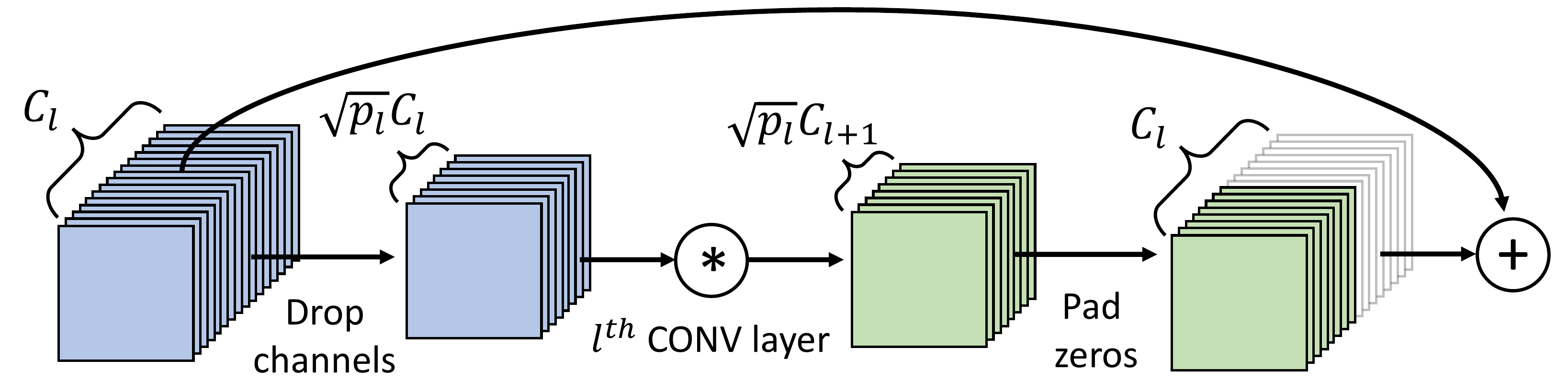}
\end{subfigure}

\caption{
Illustration of \cp for layers with residual connections --- \small{$C_{l}$ and $C_{l+1}$ represent the number of input channels of layer $l$ and $l+1$, respectively. $p_l$ represents the density of layer $l$}.
}
\label{fig:precrop}
\end{figure}

%% file: _s5_results.tex
\section{Evaluation}
In this section, we empirically evaluate \cp and \reconf. 
We first demonstrate the effectiveness of \cp by comparing it with SynFlow.
We then use \reconf to tune the width of each layer and compare the accuracy of the model after \reconf with the original model.
We perform experiments using various modern CNN models,
including ResNet~\cite{he2016deep}, MobileNetV2~\cite{sandler2018mobilenetv2}, and EfficientNet~\cite{tan2019efficientnet}, on both CIFAR-10 and ImageNet.
We set all hyperparameters used to train the models pruned by different PAI algorithms to be the same.
See Appendix~\ref{app:exp} for detailed hyperparameter settings.

\subsection{Evaluation of \cp}
\input{figures/cifar}
For CIFAR-10, we compare the accuracy of SynFlow (\textcolor{red}{red line}) and two variants of \cp: \optparam (\textcolor{blue}{blue line}) and \optflop (\textcolor{Green}{green line}).
\optparam adds the parameter count constraint whereas \optflop imposes the FLOPs constraint into the convex optimization problem.
As shown in \Figref{fig:cifar-a}, \optparam achieves similar or even better accuracy as SynFlow under a wide range of different model size constraints,
thus validating that \optparam can be as effective as the fine-grained \pai method.
Considering the benefits of structured pruning, \optparam should be favored over existing \pai methods.
\Figref{fig:cifar-b} further shows that \optflop consistently outperforms SynFlow by a large margin, especially when the reduction in FLOPs is large.
The experimental results show that \optflop should be adopted when the performance of the model is limited by the computational cost. 

\input{tables/imagenet}
Table~\ref{tab:imagenet} summarizes the comparison between \cp and SynFlow on ImageNet.
For ResNet-34, \cp achieves 0.6\% lower accuracy compared to SynFlow with a similar model size and FLOPs.
For both MobileNetV2 and EfficientNetB0, \cp achieves 1.2\% and 0.9\% accuracy improvements compared to SynFlow with strictly fewer FLOPs and parameters, respectively.
The experimental results on ImageNet further support \ourthm that a coarse-grained structured pruning (e.g., \cp) can perform as well as unstructured pruning.
In conclusion, \cp achieves a favorable accuracy and model size/FLOPs tradeoff compared to the state-of-the-art \pai algorithm.

\subsection{Evaluation of \reconf}
As discussed in \Secref{sec:reconf},
\reconf can be viewed as an ultra-fast NAS technique, which adjusts the width of each layer in the model even before the weights are initialized.
\input{tables/reconf}

Table~\ref{tab:reconf} compares the accuracy of the reconfigured model with the original model under similar model size and FLOPs constraints.
For ResNet34, with similar accuracy, we reduce the parameter count by 25\%.
For MobileNetV2, we achieve $0.3\%$ higher accuracy than the baseline with $20\%$ fewer parameters and $3\%$ fewer FLOPs.
For the EfficientNet, we can also achieve $0.3\%$ higher accuracy than the baseline with only $80\%$ of the parameters and the same FLOPs.
Note that EfficientNet is identified by a NAS method.
As \reconf only changes the number of channels of the model before initialization,
we believe it also applies to other compression techniques.

%% file: figures/cifar.tex

\begin{figure}[t]
\centering
\begin{subfigure}[H]{\textwidth}
    \includegraphics[width=\linewidth]{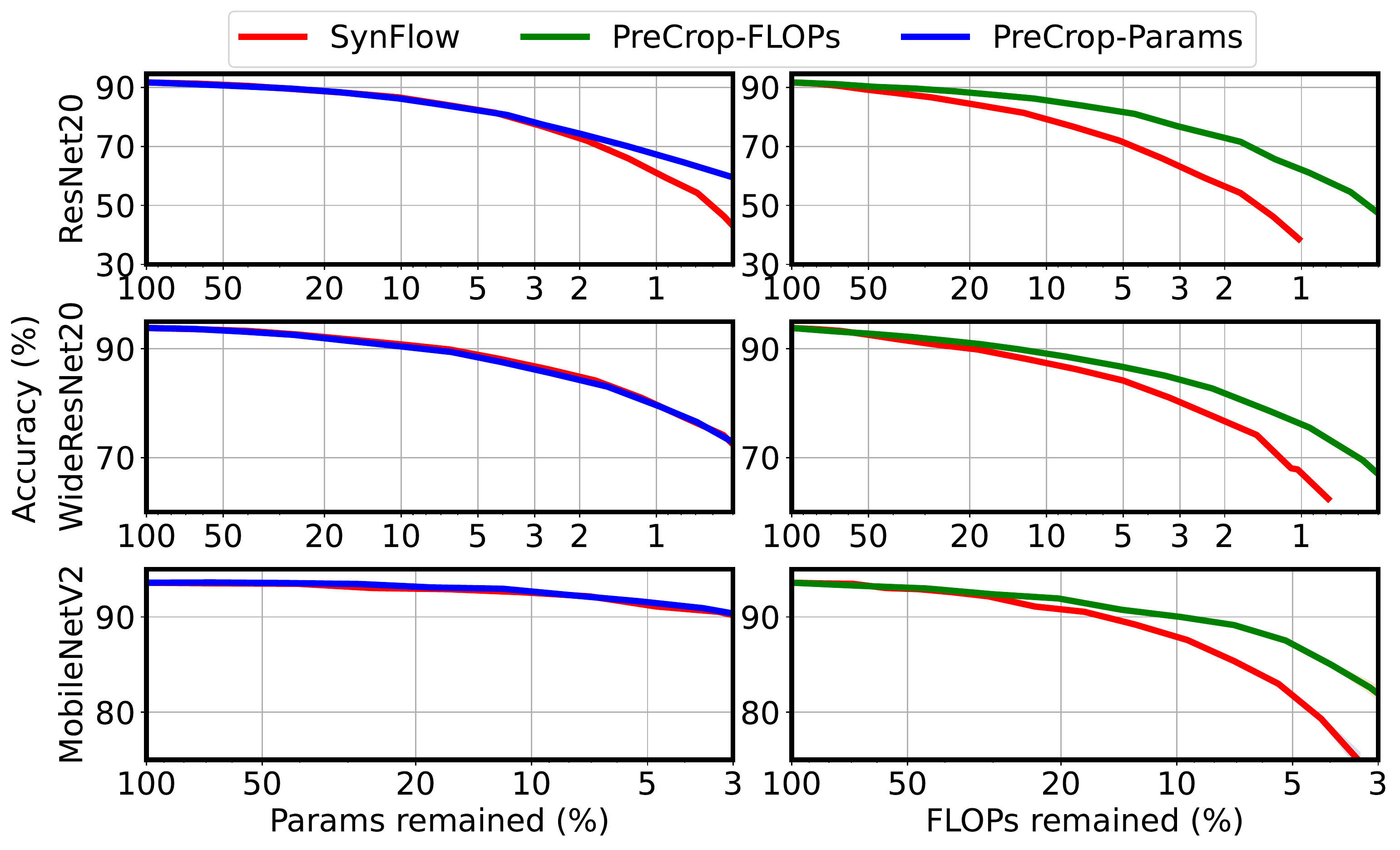}
    \vspace{-0.9cm}\\
    \subfloat[\label{fig:cifar-a}PreCrop-Params vs. SynFlow.]{\hspace{.6\linewidth}}
    \subfloat[\label{fig:cifar-b}PreCrop-FLOPs vs. SynFlow.]{\hspace{.35\linewidth}}
\end{subfigure}

\caption{
Comparison of \optparam and \optflop with SynFlow --- \small{we repeat the
experiment using ResNet20 (top), WideResNet20 (middle), and MobileNetV2 (bottom) on CIFAR-10 three times and report the mean and variance (error bar) of the accuracy}.
}
\label{fig:cifar}
\end{figure}

%% file: tables/imagenet.tex
\begin{table}[ht]
\begin{center}
\begin{small}
\begin{sc}

\vspace{-8pt}
\caption{
\textbf{Comparison of \cp with SynFlow on ImageNet} ---
The dagger($^\dagger$) implies that the numbers are only theoretical without considering the overhead of storing sparse matrices.
} 
\label{tab:imagenet}
\begin{tabular}{llccc}

\toprule
Network&Methods                & FLOPs (G)     & Params (M)     & Accuracy (\%) \\

\midrule 
\multirow{4}*{ResNet34}
&Baseline               &3.64            & 21.80             & 73.5\\

\cmidrule{2-5}
&SynFlow                &2.78$^\dagger$\footnotesize{(76.4\%)} & \textbf{10.91$^\dagger$\footnotesize{(50.0\%)}}    &\textbf{72.1 \footnotesize{(-1.4)}} \\
&\cp                    &\textbf{2.73~\footnotesize{(75.0\%)}} & 11.09 \footnotesize{(50.8\%)}    &71.5~\footnotesize{(-2.0)}\\
\midrule

\multirow{6}*{MobileNetV2}
&Baseline            &0.33          &3.51           &69.6 \\

\cmidrule{2-5}
&SynFlow                &\textbf{0.26$^\dagger$\footnotesize{(78.8\%)}}  &2.44$^\dagger$\footnotesize{(68.6\%)}  &67.6~\footnotesize{(-2.0)} \\
&\cp                    &\textbf{0.26~\footnotesize{(78.8\%)}}  &\textbf{2.33 \footnotesize{(66.4\%)}} &\textbf{68.8~\footnotesize{(-0.8)}} \\

\cmidrule{2-5}
&SynFlow                &\textbf{0.21$^\dagger$\footnotesize{(63.6\%)}}  &1.91$^\dagger$\footnotesize{(54.4\%)}  &64.5~\footnotesize{(-5.1)} \\
&\cp                    &\textbf{0.21~\footnotesize{(63.6\%)}}  &\textbf{1.85 \footnotesize{(52.7\%)}} &\textbf{67.2~\footnotesize{(-2.4)}} \\

\midrule 
\multirow{4}*{EfficientNetB0}
&Baseline               &0.40            &5.29           &73.0 \\

\cmidrule{2-5}
&SynFlow                &\textbf{0.30$^\dagger$\footnotesize{(75.0\%)}}  &3.72$^\dagger$\footnotesize{(70.3\%)}  &71.8~\footnotesize{(-1.2)} \\
&\cp                    &\textbf{0.30~\footnotesize{(75.0\%)}}  &\textbf{3.67 \footnotesize{(69.4\%)}} &\textbf{72.7~\footnotesize{(-0.3)}} \\

\bottomrule
\end{tabular}

\end{sc}
\end{small}
\end{center}
\vspace{-5pt}
\end{table}

%% file: tables/reconf.tex
\begin{table}[ht]
\begin{center}
\begin{small}
\begin{sc}

\vspace{-8pt}
\caption{
\textbf{\reconf on ImageNet.}
} 
\label{tab:reconf}
\begin{tabular}{llccc}

\toprule
Network &Methods                     & FLOPs (G)     & Params (M)     & Accuracy (\%) \\

\midrule 
\multirow{2}*{ResNet}
&Baseline &\textbf{3.64}            & 21.80             & \textbf{73.5}\\
&\reconf &\textbf{3.64} &\textbf{16.52\footnotesize{(75.8\%)}}   &73.3\footnotesize{(-0.2)}\\

\midrule
\multirow{2}*{MobileNetV2}
&Baseline &0.33          &3.51           &69.6 \\
&\reconf &\textbf{0.32~\footnotesize{(97.0\%)}} &\textbf{2.83\footnotesize{(80.6\%)}}    &\textbf{69.9\footnotesize{(+0.3)}}\\

\midrule
\multirow{2}*{EfficientNetB0}
&Baseline &\textbf{0.40} &5.29           &73.0 \\
&\reconf &\textbf{0.40} &\textbf{4.29\footnotesize{(81.1\%)}}    &\textbf{73.3(+0.3)}\\

\bottomrule
\end{tabular}

\end{sc}
\end{small}
\end{center}
\vspace{-5pt}
\end{table}

%% file: _s6_conclusion.tex
\section{Conclusion}
In this work, we show theoretically and empirically that only the layer-wise density matters for the accuracy of the CNN models pruned using \pai methods.
We formulate \pai as a simple convex \goal optimization.
Based on \goal optimization, we further propose \cp and \reconf to prune and reconfigure CNNs in the channel dimension.
Our experimental results demonstrate that \cp can outperform existing fine-grained \pai methods on various networks and datasets.

%% file: _s7_appendix.tex
\appendix
\input{proof}

\section{Solution of the Optimization problem}
For the convex optimization problem in \Eqref{eq:opt_params}, \Eqref{eq:opt_both}, or \reconf,
we can simply use Karush–Kuhn–Tucker (KKT) conditions to analytically solve it.
We include the solutions as follows for completeness.
\label{app:solve}
\subsection{Optimization with Parameter Count Constraint}
\begin{align}
\begin{split}
& p_l = \min(\frac{\mu}{\alpha}, 1) \\
\text{where $\mu$ satisfies: } &\sum_{l=1}^N \min(\alpha_l, \mu) = B_\text{params}
\end{split}
\end{align}

\subsection{Optimization with Parameter Count and FLOPs Constraints}
\begin{align}
\begin{split}
& p_l = \min(\frac{1}{\mu_1\alpha_l + \mu_2\beta_l}, 1)\\
\text{where $\mu_1, \mu_2$ satisfy: }
&\sum_{l=1}^N \alpha_l \min(\frac{1}{\mu_1\alpha_l + \mu_2\beta_l}, 1) = B_\text{params}\\
&\sum_{l=1}^N \beta_l \min(\frac{1}{\mu_1\alpha_l + \mu_2\beta_l}, 1) = B_\text{flops}
\end{split}
\end{align}

\subsection{Formulation of \reconf}
\begin{align}
\begin{split}
& p_l = \frac{1}{\mu_1\alpha_l + \mu_2\beta_l}\\
\text{where $\mu_1, \mu_2$ satisfy: }
&\sum_{l=1}^N \alpha_l \frac{1}{\mu_1\alpha_l + \mu_2\beta_l} = B_\text{params}\\
&\sum_{l=1}^N \beta_l \frac{1}{\mu_1\alpha_l + \mu_2\beta_l} = B_\text{flops}
\end{split}
\end{align}

In practice, to avoid solving the $\mu$, 
we use a convex optimization solver,
which can obtain the solution with a CPU within a second for such a small scale convex optimization.  

\section{More Empirical Results on \ourthm}
We show more empirical results that validates \ourthm.
We first show the comparison of the performance using different pruning granulariteis on VGG16 using CIFAR-10.
All the settings in this experiment is the same as in Figure~\ref{fig:thm_comp}, except this experiment is done on VGG16.
\label{app:val}
\begin{figure}[t]
\centering
\begin{subfigure}[H]{0.5\textwidth}
    \includegraphics[width=\linewidth]{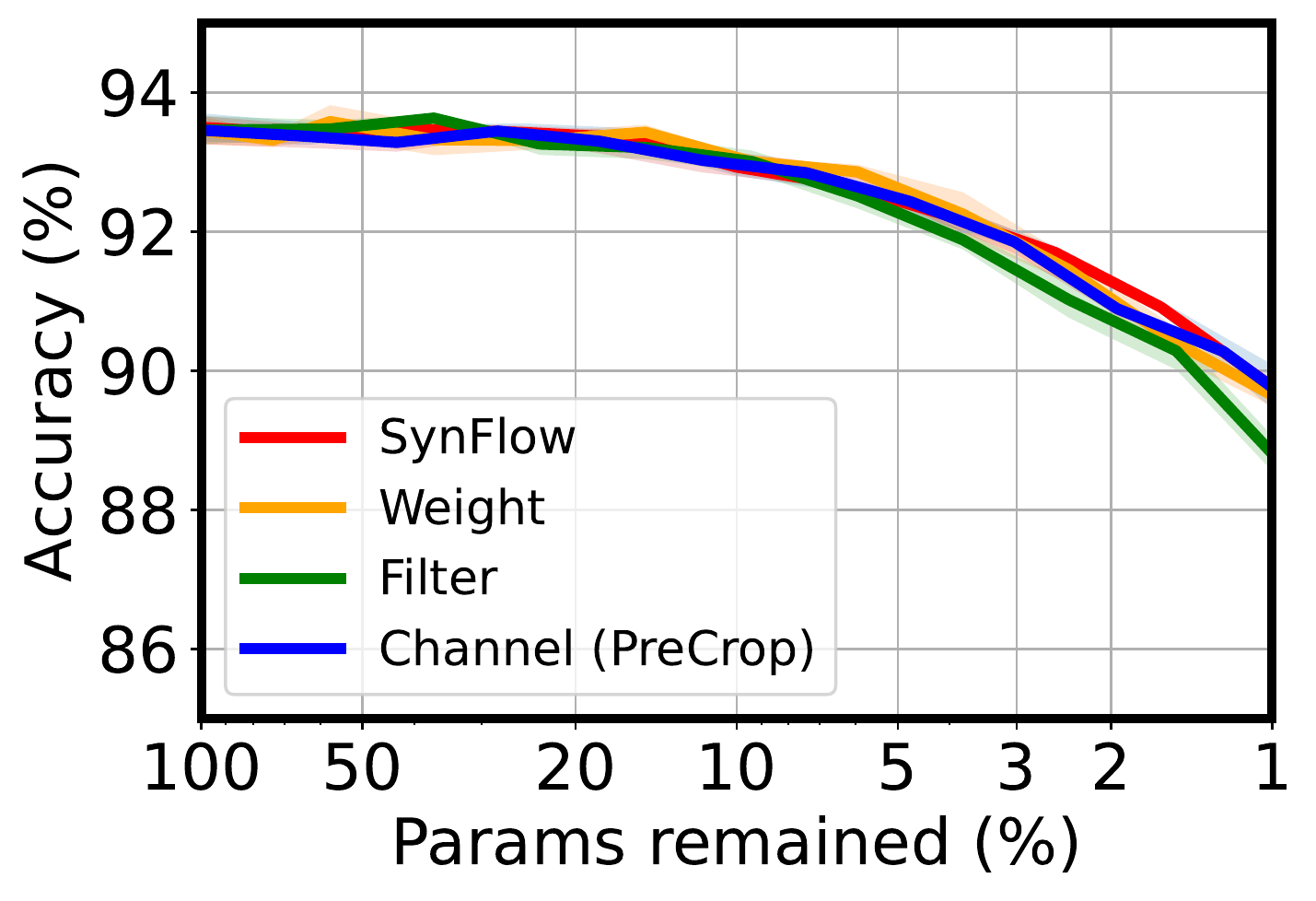}
\end{subfigure}

\caption{
Comparison of the performance using different pruning granularities on VGG16 using CIFAR-10.
}
\label{fig:thm_vgg}
\end{figure}

Then we also verify that \ourthm not only holds for SynFlow, but also holds for other \pai algorithms.
In this experiment, we first use other \pai (i.e., SNIP and GraSP) to obtain the layerwise density $p_l$.
Then we use random pruning to match $p_l$ in the channel level.
The results are shown in Figure~\ref{fig:thm_more}.
\begin{figure}[t]
\centering
\begin{subfigure}[H]{0.49\textwidth}
    \includegraphics[width=\linewidth]{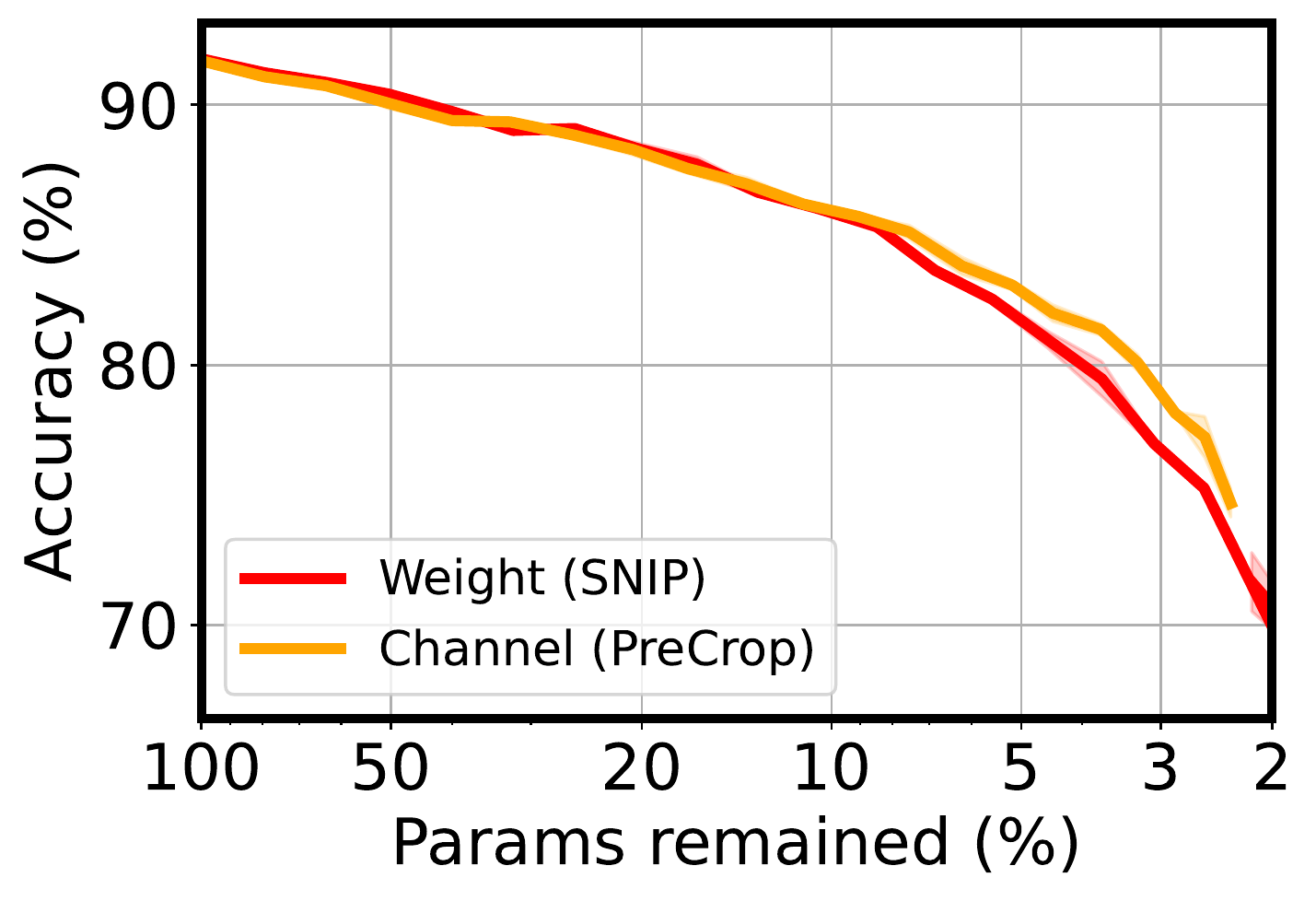}
\end{subfigure}
\begin{subfigure}[H]{0.49\textwidth}
    \includegraphics[width=\linewidth]{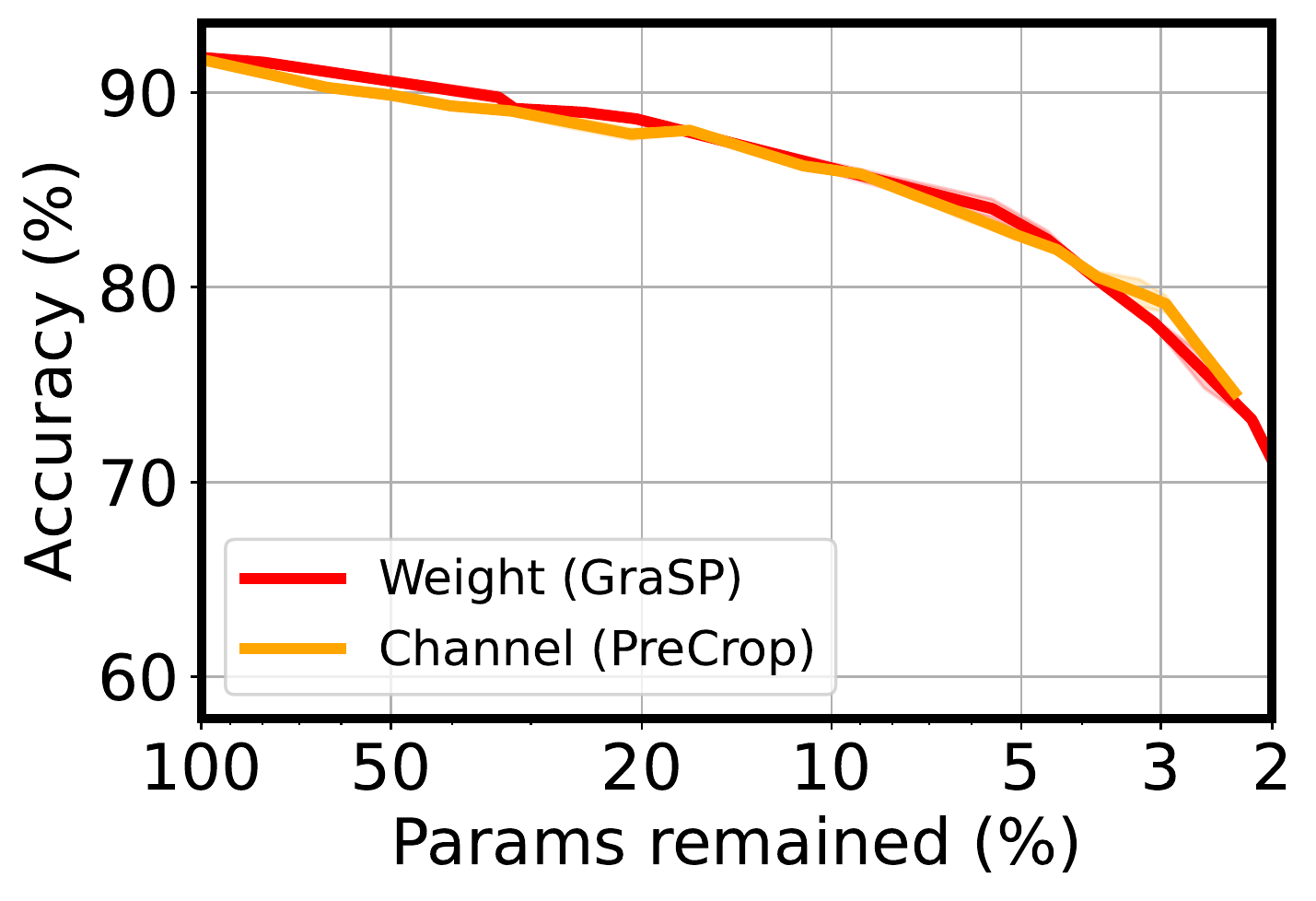}
\end{subfigure}

\caption{
Comparison of the performance using different pruning granularities on ResNet20 using CIFAR-10.
SNIP (left) and GraSP (right) importance scores are used.
}
\label{fig:thm_more}
\end{figure}

As shown in all the above experiments, 
as long as the layerwise density is the same, 
the pruning granularties do not affect the model accuracy.

\section{Channel Width Comparison}
We also include a comparison of the channel width between the baseline EfficientNetB0 and \reconf EfficientNetB0 in Figure~\ref{fig:width}.
\begin{figure}[t]
\centering
\begin{subfigure}[H]{0.5\textwidth}
    \includegraphics[width=\linewidth]{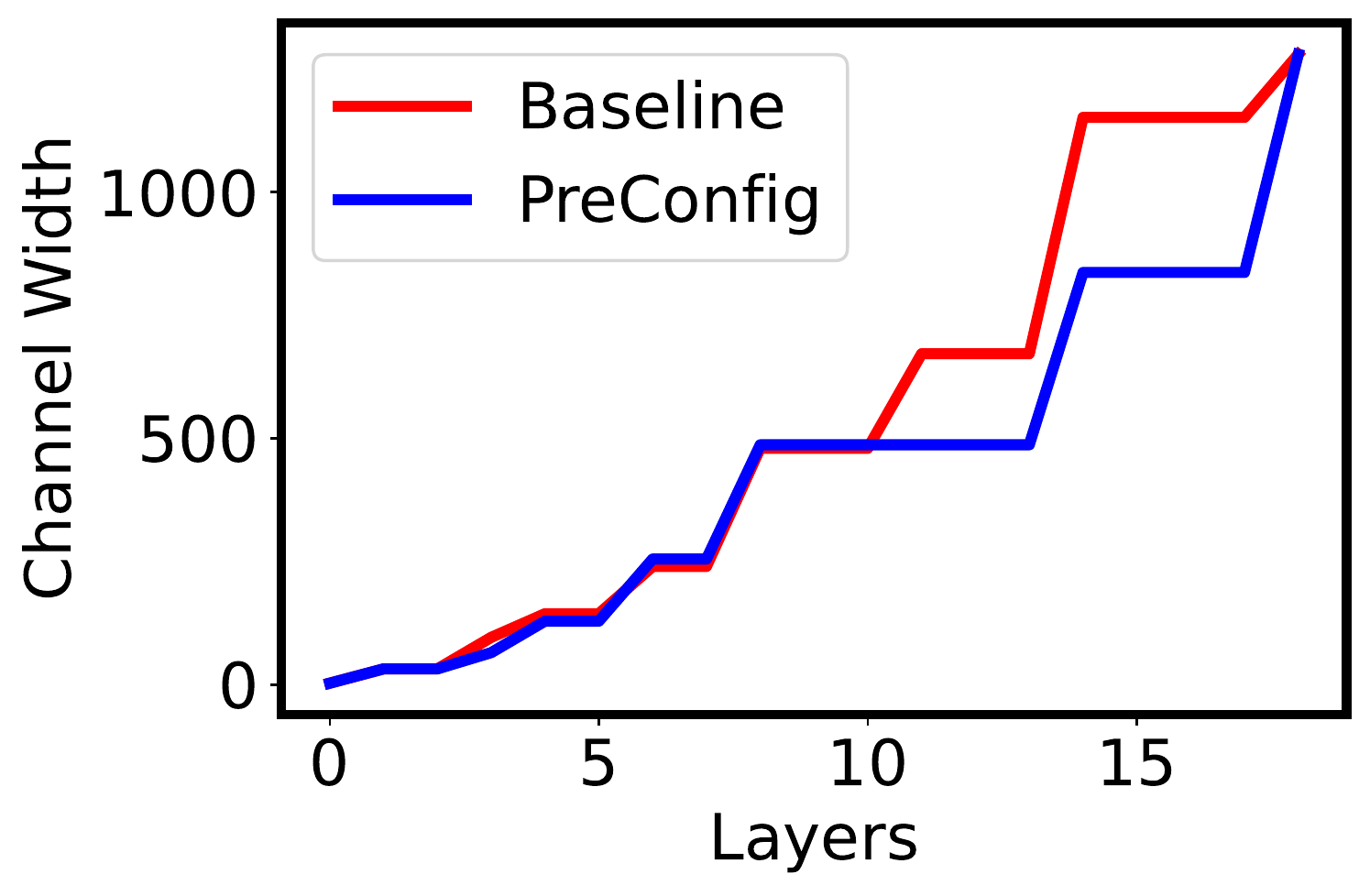}
\end{subfigure}

\caption{
Comparison of the channel width of EfficientNetB0 before and after \reconf.
}
\label{fig:width}
\end{figure}

\section{Experiment Details}
\label{app:exp}
\subsection{Implementation}
We adapt model implementations of ResNet, ShuffleNet, and MobileNetv2 from imgclsmob\footnote{https://github.com/osmr/imgclsmob}. 
The implementations of SynFlow, SNIP, and GraSP are based on the codebase of SynFlow\footnote{https://github.com/ganguli-lab/Synaptic-Flow}.

\subsection{Hyperparameters}
Here we provide the hyperparameters used in training all models in Table~\ref{table:hyper}.
No AutoAugment, Label Smoothing, or stochastic depth is used during training.
All the CIFAR-10 models are trained with same hyperparameter setting.
\input{tables/hyper}

%% file: proof.tex
\section{Proof of \ourthm}
\label{app:proof}
\begin{theorem}
\label{thm}
Given a specific CNN architecture, the \goal ($\E_{[M, W]} [\mathcal{S}_\text{SF}]$) of any randomly compressed model with the same layer-wise density $p_l$ is a constant, independent of the pruning granularity.
The constant \goal equals to:
\begin{equation}
\mathop{\E}_{M, W}[\mathcal{S}_\text{SF}] 
= N C_{N+1}\prod_{l=1}^N (p_l C_{l} \cdot \E_{x\sim \mathcal{D}}[|x|]) \ ,
\end{equation}
where $N$ is the number of layers in the network,
$\E_{x\sim \mathcal{D}}[|x|]$ is the expectation of magnitude of distribution $\mathcal{D}$,
$C_l$ is the input channel size of layer $l$ and is also the output channel size of $l-1$,
and $p_l=\frac{1}{\alpha_l} \|M_l\|_0$ is the layer-wise density.
\end{theorem}
\begin{proof}
Assuming the network has $N$ layers, weight matrix $W^l \in \R^{m_l \times n_l}$,
mask matrix $M^l \in \{0, 1\}^{C_l \times C_{l+1}}$.
$C_l$ and $C_{l+1}$ are the input and output channel size of layer $l$. 
As the output channel size of any layer $l$ equals to the input channel size of the next layer $l+1$, we have $C_{l+1} = C_{l+1}, \forall l<N$. 

We first prove the Theorem \ref{thm} on fully-connected network, and we can extend it to CNNs easily.
From \Eqref{eq:synflow}, in a fully-connected network,
the Synaptic Flow score for any parameter $W^l_{ij}$ with mask $M^l_{ij}$ in layer $l$ equals to:
\begin{align}
\begin{split}
\label{eq:loss}
    \mathcal{S}_{\text{SF}}(W^l_{(i,j)})
    &=\left[\mathds{1}^T\prod_{k=l+1}^{N}\left|W^k\odot M^k\right|\right]_i
    \left|W^l_{(i, j)} M^l_{(i,j)}\right| 
    \left[\prod_{k=1}^{l-1} \left|W^k\odot M^k\right|\mathds{1}\right]_j 
\end{split}
\end{align}

We compute the \goal of the layer $l$ ($\E_{[M, W]}(\mathcal{S}_\text{SF})^{[l]}$),
then the \goal of the network is simply the sum of \goal of all layers:
\begin{equation}
\label{eq:sum}
    \E_{[M, W]}(\mathcal{S}_\text{SF}) = \sum_{l=1}^N \E_{[M, W]}(\mathcal{S}_\text{SF})^{[l]}
\end{equation}
We define the expectation value for input channel $i$, output channel $j$, and the whole layer in layer $l$ as $E^l_{(i, *)}$, $E^l_{(*, j)}$, and $E^l_{(*, *)}$:
\begin{align}
    &\E^l_{(i, *)} = \frac{1}{C_{l+1}} \sum_{x}|W^l_{(i, x)} M^l_{(i ,x)}| \\
    &\E^l_{(*, j)} = \frac{1}{C_l} \sum_{x}|W^l_{(x, j)} M^l_{(x ,j)}| \\
    &\E^l_{(i, j)} =\E^l_{(*, *)} = \frac{1}{C_lC_{l+1}} \sum_{i, j}|W^l_{(i, j)} M^l_{(i ,j)}| 
    = \frac{1}{\alpha_l} \sum_{i, j}|W^l_{(i, j)} M^l_{(i ,j)}|  
    = p_l \E_{|\mathcal{D}^l|} 
\end{align}
Here we use $\E_{|\mathcal{D}^l|}$ to denote $\E_{x\sim \mathcal{D}}[|x|]$.

As the weight in layer $l$ is sampled from distribution $\mathcal{D}$,
and the mask matrices are also randomly sampled,
we have
\begin{equation}
    \E^{[k]}_{(*, *)} 
    = \E^k_{(i, *)} 
    = \E^k_{(*, j)} 
    = p_l \E_{|\mathcal{D}^l|}
\end{equation}

With $E^k_{(i, *)}$, $E^k_{(*, j)}$, and $E^l_{(*, *)}$, we can rewrite \Eqref{eq:loss} to:

\begin{align}
\begin{split}
\label{eq:loss_f}
    \E[\mathcal{S}_{\text{SF}}(W^l_{(i,j)})] &= 
    \left( \prod_{k=l+2}^{N} C_{k+1} \E^k_{(*, *)} \right) \cdot
    C_{l+2}\E^{l+1}_{(i, *)} \cdot
    \E^{l}_{(i, j)} \cdot
    C_{l-1}\E^{l-1}_{(*, j)} \cdot
    \left( \prod_{k=1}^{l-2} C_k E^k_{(*, *)} \right) 
\end{split}
\end{align}

Combining Equation \ref{eq:expectation} and \ref{eq:loss_f},
because the instantiation of the weight matrices and mask matrices for each layer are independent:
\begin{align}
\begin{split}
    &\E_{[M, W]}(\mathcal{S}_\text{SF})^{[l]} = 
    \E \left[ \sum_{i=1}^{C_l}\sum_{j=1}^{C_{l+1}}\mathcal{S}_{\text{SF}}(W^l_{(i,j)}) \right]
    = \sum_{i=1}^{C_l}\sum_{j=1}^{C_{l+1}} \E\left[\mathcal{S}_{\text{SF}}(W^l_{(i,j)}) \right] \\
    & = \left( \prod_{k=l+2}^{N} p_k C_{k+1} \E_{|\mathcal{D}^k|} \right) 
    \sum_{i=1}^{C_l}\sum_{j=1}^{C_{l+1}} \left(
    p_{l+1} C_{l+2}\E^{l+1}_{(i, *)} \cdot
    p_{l} \E_{|\mathcal{D}^l|} \cdot
    p_{l-1}C_{l-1}\E^{l-1}_{(*, j)} \right)
    \left( \prod_{k=1}^{l-2} C_k \E_{|\mathcal{D}^k|} \right)  \\
    & = \left( \prod_{k=l+2}^{N} p_k C_{k+1} \E_{|\mathcal{D}^k|} \right) 
    C_l C_{l+1}\left(
    p_{l+1}C_{l+2}\E_{|\mathcal{D}^{l+1}|} \cdot
    p_{l}\E_{|\mathcal{D}^l|} \cdot
    p_{l-1}C_{l-1}\E_{|\mathcal{D}^{l-1}|} \right)
    \left( \prod_{k=1}^{l-2} p_k C_k \E_{|\mathcal{D}^k|} \right)  \\
    & = C_{N+1}\prod_{l=1}^N (p_lC_{l} \E_{|\mathcal{D}^l|})
\end{split}
\end{align}

According to \Eqref{eq:sum}, 
\begin{align}
\begin{split}
\E_{[M, W]}(\mathcal{S}_\text{SF}) 
&= \sum_{l=1}^N C_{N+1}\prod_{l=1}^N (C_{l} \E_{|\mathcal{D}^l|}) \\
&= N C_{N+1}\prod_{l=1}^N (p_l C_{l} \cdot \E_{x\sim \mathcal{D}}[|x|]),
\end{split}
\end{align}

\ourthm can also be extended to CNNs, 
as it is obvious that \goal of CNNs is proportional to that of fully connected networks.
Thus the difference of \goal between CNNs and fully connected networks for each layer is only a factor equal to $\frac{K^2}{C_{l+1}}$, 
where $K$ is the kernel size of the convolutional layer.
\end{proof}

%% file: tables/hyper.tex
\begin{table}[ht]
\centering
\caption{Hyperparameters used in training.}
\vspace{10pt}
\label{table:hyper}
\small
\begin{tabular}{rrrrrrr}
\toprule
&  \multirow{2.5}*{CIFAR-10}         &\multicolumn{3}{c}{ImageNet}  \\
\cmidrule{3-5}
&                                           & MobileNet         & ResNet                & EfficientNet\\
\midrule        
Optimizer               & momentum          & momentum          & momentum              & momentum\\ 
Training Epochs         & 160               & 180               & 90                    & 150\\ 
Batch Size              & 128               & 256               & 512                   & 256\\ 
Initial Learning Rate   & 0.1               & 0.025             & 0.2                   & 0.035\\ 
Learning Rate Schedule  & linear            & drop at each epoch& drop at 30, 60 epoch  & drop at each epoch\\ 
Drop Rate               & N.A.              & 0.98              & 0.1                   & 0.99\\ 
Weight Decay            & $10^{-4}$         & $4\times 10^{-5}$ & $10^{-4}$             & $4 \times 10^{-5}$\\ 
\bottomrule
\end{tabular}
\end{table}